\documentclass[letterpaper, 10 pt, conference]{ieeeconf}  
\usepackage{times}

\usepackage{amsthm, amsfonts, amsmath, amssymb, mathtools}
\usepackage{multicol}
\usepackage{xcolor}
\usepackage{multirow}
\usepackage{booktabs}
\usepackage[bookmarks=true, hidelinks]{hyperref}

\IEEEoverridecommandlockouts                              

\title{\LARGE \bf
Cooptimizing Safety and Performance Using Safety Value-Constrained Model Predictive Control
}

\author{Hao Wang$^{1, 2}$, Nam Nguyen$^{3}$, Armand Jordana$^{3}$, Ludovic Righetti$^{3}$, and Somil Bansal$^{2}$ 
\thanks{This research is supported in part by the DARPA Assured Neuro Symbolic Learning and Reasoning (ANSR) program and by the NSF CAREER program (2240163).}%
\thanks{$^{1}$Hao Wang is with the Ming Hsieh Department of Electrical and Computer Engineering, University of Southern California.
        {\tt\small haowwang@usc.edu}}%
\thanks{$^{2}$The authors are with the Department of Aeronautics and Astronautics, Stanford University.
        {\tt\small \{haowwang, somil\}@stanford.edu}}%
\thanks{$^{3}$The authors are with the Electrical and Computer Engineering Department, New York University.
        {\tt\small \{namnguyen, aj2988, ludovic.righetti\}@nyu.edu}}%
}

\newcommand{\state}{x}
\newcommand{\sset}{\mathcal{X}}
\newcommand{\sdim}{n_{\state}}
\newcommand{\stateseq}{\mathbf{\state}}
\newcommand{\ctrl}{u}
\newcommand{\ctrlseq}{\mathbf{\ctrl}}
\newcommand{\cset}{\mathcal{U}}

\newcommand{\controller}{\pi}
\newcommand{\cdim}{n_{\ctrl}}

\newcommand{\tfunc}{l}

\newcommand{\safeset}{\mathcal{S}}

\newcommand{\vfunc}{V}

\newcommand{\vfsafe}{\vfunc_s}
\newcommand{\cost}{J}

\newcommand{\traj}{\xi} 
\newcommand{\ctrlconstr}{g}

\newcommand{\tvar}{t}
\newcommand{\tend}{T}
\newcommand{\tdummy}{\tau}
\newcommand{\tinit}{0}
\newcommand{\thor}{[\tinit,\tend]}
\newcommand{\ph}{h}
\newcommand{\ch}{h_c}
\newcommand{\tendmpc}{K}
\newcommand{\mpcindex}{k}
\newcommand{\mpctvar}{j}
\newcommand{\dyn}{f}
\newcommand{\ddyn}{f_d}

\newcommand{\reals}{\mathbb{R}}

\newtheorem{problem}{Problem}

\newtheorem{proposition}{Proposition}

\newtheorem{remark}{Remark}

\begin{document}

\maketitle
\thispagestyle{empty}
\pagestyle{empty}

\renewcommand{\thefootnote}{*}
\begin{abstract}
    Autonomous systems are increasingly deployed in real-world environments, where they must achieve high performance while maintaining safety under state and input constraints. Although Model Predictive Control (MPC) provides a principled framework for constrained optimal control, guaranteeing safety beyond its finite planning horizon remains a fundamental challenge.
    In this work, we augment MPC with a safety value function–based terminal constraint that enforces membership in a control-invariant safe set at the end of each planning horizon. This formulation enables real-time synthesis of trajectories that are both high-performing and provably safe. 
    We show that, under an exact safety value function and a feasible initialization, the proposed MPC scheme is recursively feasible, thereby ensuring persistent safety.
    In contrast to traditional terminal set constructions that rely on local linearizations or conservative approximations, our approach incorporates a reachability-based safety value function for terminal constraints, yielding less conservative and more expressive safety guarantees.
    We validate the proposed framework through simulation and hardware experiments on a Flexiv Rizon 10s manipulator. Results demonstrate improved constraint satisfaction and robustness compared to standard state-constrained MPC and reactive safety filtering, while maintaining competitive task performance. The full implementation and experiments are available on the project website\footnote{\url{https://github.com/haowwang/safety_value_mpc}}.
\end{abstract} 

\section{Introduction}
Recent advances in robotics \cite{pi06start, gemini_robotics_1_5} have brought autonomous systems closer than ever to performing meaningful tasks in real-world environments. However, successful deployment of robots requires controllers that achieve high task performance while rigorously satisfying safety constraints on states and inputs—for example, manipulators operating near obstacles while carrying payloads or under degraded actuator limits. A central challenge is that safety and performance are often treated in a disconnected manner: controllers are typically optimized for performance at design time, while safety is enforced later as an auxiliary mechanism.

Safety filtering is a prominent example of this design philosophy \cite{ames_2016_cbf_qp, borquez_2024_safety_filtering, wabersich_2023_data_driven_filter}, where a filter modifies the output of a nominal controller to ensure constraint satisfaction. Although such approaches provide formal safety guarantees, they treat safety as a corrective layer rather than a co-optimized objective, often leading to myopic or overly conservative behavior. In contrast, the optimal control community has developed principled methods for synthesizing controllers under state-constrained optimal control formulations \cite{soner1986_state_constrained_oc, capuzzo1990_hj_state_cons, altarovici_2013_scocp, wang_2024_coop_safe_perf}. These approaches unify safety and performance at the formulation level, but their scalability is limited, as they typically rely on reachability or dynamic programming techniques that become intractable for high-dimensional systems.

In this work, we seek a scalable framework that co-optimizes safety and performance at design time for high-dimensional systems. We adopt the state-constrained optimal control formulation and solve it online using Model Predictive Control (MPC), leveraging MPC’s closed-loop replanning to enhance robustness to disturbances and modeling errors. However, in practical implementations, the MPC planning horizon is limited by real-time computational requirements. A finite horizon can compromise recursive feasibility and ultimately lead to safety violations, since constraint satisfaction is not explicitly guaranteed beyond the planning window.

To address this challenge, we introduce a value-function–based terminal constraint: at each planning step, we require the predicted terminal state to lie within a control-invariant safe set characterized by a safety value function. This construction ensures that the system can be kept safe beyond the horizon through a certified backup policy. To enable applicability to high-dimensional manipulators with complex state and input constraints, we approximate the safety value function using learning-based Hamilton–Jacobi (HJ) reachability methods.
This choice of a reachability-based safety value function offers two key advantages. First, in contrast to conservative terminal set constructions that rely on local approximations or handcrafted invariant sets, HJ reachability analysis characterizes the maximal control invariant safe set under the given dynamics and constraints \cite{hjr_survey}. Consequently, when used as a terminal constraint, it yields the largest possible feasible set for the MPC problem without sacrificing safety, leading to a better performance objective optimization. 

Second, reachability analysis naturally accommodates complex state and input constraints, including nonlinear dynamics, actuator limits, obstacle avoidance, and coupled constraints that commonly arise in robotic systems. This expressiveness allows our framework to enforce safety in a wide range of manipulation and interaction scenarios without simplifying the constraint structure.
Overall, our framework combines the global safety certification of reachability analysis with the computational tractability and performance optimization of MPC, yielding a scalable approach for real-time, safety-critical control of high-dimensional robotic systems.
The core contributions of this work are as follows
\begin{enumerate}
    \item We present an MPC formulation with a safety value function–based terminal constraint that ensures persistent constraint satisfaction and recursive feasibility, while optimizing performance. Under an exact safety value function and feasible initialization, we prove recursive feasibility of the closed-loop scheme.
    \item We synthesize a safety value function for a high dimensional system, namely a 14D manipulator, with complex state and control constraints using learning-based HJ reachability.
    \item We demonstrate the effectiveness of the proposed framework through extensive simulation studies across randomized trials and validate it experimentally on a Flexiv Rizon 10s manipulator under payload and obstacle-avoidance constraints.
\end{enumerate}

\section{Related Work}
\noindent \textbf{State-Constrained Optimal Control.} 
Designing controllers that are both safe and high-performing has long been a central problem in control theory. The primary theoretical framework for studying this problem is the state-constrained optimal control formulation, in which safety requirements are encoded as state constraints and task objectives are captured in the cost functional. Early foundational works \cite{soner1986_state_constrained_oc, capuzzo1990_hj_state_cons} established existence and characterization results under restrictive controllability and regularity assumptions. More recent efforts \cite{altarovici_2013_scocp, wang_2024_coop_safe_perf} have relaxed some of these assumptions and developed computational methods for specific classes of problems. However, these approaches fundamentally rely on dynamic programming and Hamilton–Jacobi–Bellman (HJB) formulations \cite{dp_book}, which suffer from the curse of dimensionality and therefore do not scale to high-dimensional robotic systems, though some data-driven efforts have been made to alleviate dimensionality challenges \cite{coop_pinn}. 

\noindent \textbf{Safety Filtering.} 
An alternative and widely adopted paradigm for ensuring safety in autonomous systems is safety filtering \cite{ames_2016_cbf_qp, borquez_2024_safety_filtering, wabersich_2023_data_driven_filter}. In this approach, a safety filter modifies the output of a nominal controller in real time to enforce constraint satisfaction, often using Control Barrier Functions (CBFs), reachability analysis, or related constructs. These methods are computationally efficient and broadly applicable across system classes. However, these methods treat safety as a corrective layer applied after the nominal controller is designed. As a result, safety and performance are not co-optimized, and the resulting behavior can be myopic, overly conservative, or suboptimal in practice.

\noindent \textbf{Model Predictive Control (MPC).} 
With advances in real-time optimization and numerical solvers, MPC has emerged as a powerful framework for synthesizing performant controllers that explicitly account for constraints. Several works \cite{cbf_mpc, zeng_enhance_safety_cbf_mpc, ihocbf, dynamic_obs_cbf} integrate Control Barrier Functions into the MPC formulation to enforce safety while preserving recursive feasibility. More broadly, terminal constraints and terminal sets have long been used in MPC to guarantee recursive feasibility and stability. The central challenge in these approaches lies in constructing a terminal set that is both control-invariant and computationally tractable. Existing CBF-based methods effectively use barrier certificates to define such terminal constraints, and recent work \cite{morton_manip_cbf} extends CBF design to manipulator systems. However, designing CBFs for high-dimensional nonlinear systems—particularly under complex input constraints—remains challenging and often requires problem-specific structure. Furthermore, the performance of the resulting controller is directly influenced by the conservativeness of the CBF used to construct the terminal constraint. Overly conservative barrier functions lead to small terminal sets, which restrict the feasible region of the MPC problem and can significantly degrade performance. This motivates the need for terminal sets that are as large as possible—ideally maximal—while still guaranteeing safety. 

\section{Problem Formulation}
In this work, we formulate the problem of co-optimizing safety and performance as a finite-horizon state-constrained optimal control problem \cite{altarovici_2013_scocp, wang_2024_coop_safe_perf} with horizon $\thor$.
Consider a continuous-time, deterministic, control-affine system governed by the ordinary differential equation $\dot{\state} = \dyn(\state,\ctrl)$
where $\state \in \sset \subseteq \reals^{\sdim}$ denotes the system state and $\ctrl \in \cset \subseteq \reals^{\cdim}$ denotes the control input. 
We denote by $\traj_{\state,\tvar}^{\ctrlseq}:[\tvar,\tend]\rightarrow\sset$ the state trajectory that evolves from an initial state $\state$ at time $\tvar$ under the control signal $\ctrlseq:[\tvar,\tend)\rightarrow \cset$. With a slight abuse of the notation, we use $\traj_{\state,\tvar}^{\ctrlseq}(\tdummy)$ to denote the system state at time $\tdummy\geq\tvar$ along this trajectory.

The performance objective is encoded through a running cost $r:\sset\times\cset\rightarrow\reals$ and a terminal cost $\phi:\sset\rightarrow\reals$ which together define the objective function in Eq.~\ref{eq:state_constrained_oc_cost}. On the other hand, the safety requirements are enforced through a \emph{state} constraint function $\tfunc:\sset\rightarrow\reals$, as specified in Eq.~\ref{eq:state_constraint}. Additionally, $\ctrlconstr:\cset\rightarrow\reals$ is the control constraint. Our goal is to synthesize a time-varying state-feedback controller $\controller : \sset \times [\tinit,\tend) \rightarrow \cset$ that solves Prob.~\ref{prob:state_constrained_opt_ctrl_prob}. 

\begin{problem}[State-Constrained Optimal Control Problem]\label{prob:state_constrained_opt_ctrl_prob}
\begin{subequations}
\begin{align}
    \begin{split}\label{eq:state_constrained_oc_cost}
    &\inf_{\ctrlseq} \quad \cost(\state, \tvar,\ctrlseq) = \int_\tvar^\tend r(\traj_{\state,\tvar}^{\ctrlseq}(\tdummy), \ctrlseq(\tdummy))  d\tdummy \\ 
    & \qquad \qquad \qquad \qquad \qquad \ \ \ + \phi(\traj_{\state,\tvar}^{\ctrlseq}(\tend)) \\
    \end{split}\\
    &s.t.  \quad \frac{d}{d\tdummy}\traj_{\state,\tvar}^{\ctrlseq}(\tdummy) = \dyn(\traj_{\state,\tvar}^{\ctrlseq}(\tdummy), \ctrlseq(\tdummy))  \ \forall \tdummy \in [\tvar,\tend) \label{eq:dyn}\\
    & \qquad \ \ctrlconstr(\ctrlseq(\tdummy))\geq 0  \ \forall \tdummy \in [\tvar,\tend) \label{eq:ctrl_constraint} \\ 
    & \qquad \ \tfunc(\traj_{\state,\tvar}^{\ctrlseq}(\tdummy))\geq 0  \ \forall \tdummy \in [\tvar,\tend] \label{eq:state_constraint}
    \end{align}
\end{subequations}
\end{problem}

\section{Background}
In this work, we solve the state-constrained optimal control problem Prob.~\ref{prob:state_constrained_opt_ctrl_prob} via MPC and use reachability-based terminal constraint. Below we provide a brief overview of MPC and HJ reachability.

\subsection{Model Predictive Control}
Model predictive control (MPC) solves the discrete-time counterpart of Prob.~\ref{prob:state_constrained_opt_ctrl_prob} for some planning horizon $\ph$ at each time step $\mpctvar < \tendmpc$, formulated in Prob.~\ref{prob:mpc}, and applies the first $c$ controls from the solution control signal. Then, at the next time step, Prob.~\ref{prob:mpc} is solved from the evolved state, and this process continues until the task horizon $\tendmpc$ is reached. We denote the planning state and control trajectory at time $\mpctvar$ by $\stateseq_\mpctvar$ and $\ctrlseq_\mpctvar$, respectively. 

\begin{problem}[State-Constrained MPC Formulation]\label{prob:mpc}
    \begin{subequations}
    \begin{align}
        &\min_{\ctrlseq} \quad \sum_{\mpcindex=0}^{\ph-1} r(\stateseq_\mpctvar(\mpcindex), \ctrlseq_\mpctvar(\mpcindex)) + \phi(\stateseq_\mpctvar(\ph)) \\ \label{eq:mpc_objective}
        \begin{split}
             & s.t. \quad \stateseq_\mpctvar(\mpcindex+1)  = \ddyn(\stateseq_\mpctvar(\mpcindex), \ctrlseq_\mpctvar(\mpcindex))  \\
             & \qquad \qquad \quad \forall \mpcindex \in \{0, \ldots, \ph - 1\} \\
        \end{split} \\
        & \qquad \ \ctrlconstr(\ctrlseq_\mpctvar(\mpcindex))\geq 0  \ \forall \mpcindex \in \{0, \ldots, \ph-1\}  \\
        & \qquad \ \tfunc(\stateseq_\mpctvar(\mpcindex))\geq 0  \ \forall \mpcindex \in \{0, \ldots, \ph\} 
        \end{align}
    \end{subequations}
\end{problem}

\subsection{Hamilton-Jacobi (HJ) Reachability Analysis}
Given a continuous-time system with bounded and Lipschitz dynamics $\dyn$, and a state constraint (Eq.~\ref{eq:state_constraint}), represented by a function $\tfunc(\state)$ intuitively measuring how close a state is to state constraint violation, the primary objective of Hamilton-Jacobi (HJ) Reachability analysis \cite{hjr_survey} is to solve the following \emph{safety optimal control problem} for all state $\state$ and time $\tvar\in\thor$. 

\begin{problem}[Safety Optimal Control Problem]\label{prob:safety_ocp}
\begin{subequations}
\begin{align}
    &\inf_{\ctrlseq} \quad \cost(\state, \tvar,\ctrlseq) = \min_{\tau\in[\tvar,\tend]} \tfunc(\traj_{\state,\tvar}^{\ctrlseq}(\tdummy)) \label{eq: reachability_objective}\\
    &s.t.  \quad \frac{d}{d\tdummy}\traj_{\state,\tvar}^{\ctrlseq}(\tdummy) = \dyn(\traj_{\state,\tvar}^{\ctrlseq}(\tdummy), \ctrlseq(\tdummy))  \ \forall \tdummy \in [\tvar,\tend) 
    \end{align}
\end{subequations}
\end{problem}
The solution of Prob.~\ref{prob:safety_ocp} at all state $\state\in\sset$ and time $\tvar\in\thor$ is represented by the \emph{safety value function} $\vfsafe(\state,\tvar)$, which is shown to be the viscosity solution of the following Hamilton-Jacobi-Bellman Variational Inequality (HJB-VI) \cite{min_time_ctrl}: 
\begin{equation} \label{eq:HJBVI}
    \begin{aligned}
    &\min \biggl\{\frac{\partial\vfsafe}{\partial \tvar} + \max_{\ctrl\in\cset} \{\frac{\partial\vfsafe}{\partial \state}^\top \dyn(\state,\ctrl)\}, \tfunc(\state) - \vfsafe(\state, \tvar) \biggr\} = 0\\
    & \forall \state\in\sset \ \text{and} \ \forall\tvar\in[\tinit,\tend),  \, \vfsafe(\state, \tend) = \tfunc(\state) \ \forall \state\in\sset
    \end{aligned}
\end{equation}
The safety value function $\vfsafe$ can be written succinctly as follows 
\begin{equation}\label{eq:safety_value_function}
    \vfsafe(\state, \tvar) = \sup_{\ctrlseq} \min_{\tdummy \in [\tvar,\tend]} \tfunc(\traj_{\state,\tvar}^{\ctrlseq}(\tdummy))
\end{equation}
In practice, the \emph{converged} safety value function defined by 
\begin{equation}\label{eq:converged_safety_value_function}
    \vfsafe(\state) = \lim_{\tvar\rightarrow\infty}\vfsafe(\state,\tvar) = \sup_{\ctrlseq} \min_{\tdummy \in [\tinit,\infty)} \tfunc(\traj_{\state,\tinit}^{\ctrlseq}(\tdummy))
\end{equation}
is an attractive alternative to the time-varying safety value function Eq.~\ref{eq:safety_value_function} as it encodes infinite-time safety information and can be used to ensure safety for arbitrary time horizons. For the remainder of this work, we rely exclusively on the converged safety value function, $\vfsafe(\state)$, which we will hereafter refer to simply as the safety value function.

It follows immediately that the super-zero level set of $\vfsafe$ is the set of states starting from which the system can maintain state constraint satisfaction indefinitely, we call this set the \emph{safe set} $\safeset$: 
\begin{equation}
    \safeset \triangleq \{\state\in\sset|\vfsafe(\state)\geq 0 \}
\end{equation}

In practice, we compute the safety value function by solving the HJB-VI either over a grid \cite{helperOC} or using neural approximators \cite{bansal_2021_deepreach, deepreach_mpc}. Grid-based methods \cite{helperOC} provide accurate solutions for low-dimensional systems (typically less than 6 dimensions) but are intractable for higher dimensional systems. One of the learning-based methods that have been shown to produce high quality solution is DeepReach \cite{deepreach_mpc}, which approximates the solution to the HJB-VI using two losses: 1) PDE residual loss, incentivizing the neural network to comply with the HJB-VI in Eq. \ref{eq:HJBVI}, and 2) supervision loss, fitting the network to known approximate solutions of the HJB-VI at different points in the state space. 
Intuitively, the PDE loss helps the neural network to learn a good approximation of the safety value function and the supervision loss helps stabilize the training. We refer the interested readers to \cite{deepreach_mpc} for more details.

\section{Method}
We solve Prob. \ref{prob:state_constrained_opt_ctrl_prob} online in a model predictive control (MPC) fashion (i.e. solving Prob. \ref{prob:state_constrained_opt_ctrl_prob} over a shorter planning horizon $\ph$, executing the first few controls, and solving Prob. \ref{prob:state_constrained_opt_ctrl_prob} again from the evolved state).
This closed-loop nature of MPC provides robustness against potential modeling errors and unforeseen disturbances in deployment. 

Although modern optimal control solvers have demonstrated impressive speed and reliability, in practice the planning horizon $\ph$ of MPC is typically much shorter than the overall task horizon $T$ due to real-time computational constraints. While this shortened horizon reduces the online computation burden, it can lead to myopic behavior, as the optimizer cannot reason about safety beyond the planning window. As a result, trajectories that appear feasible within the horizon may eventually lead to violations of safety constraints.

The key idea of our approach is to incorporate the HJ reachability-based safety value function into the MPC formulation, ensuring that the predicted terminal state lies within a certified safe set. This guarantees that safety can be maintained beyond the planning horizon, enabling persistent satisfaction of safety constraints during closed-loop operation.

\subsection{Safety Value MPC}\label{subsec:sv_mpc}

A desirable terminal constraint should satisfy two key properties: (a) its feasible set should be control invariant to ensure recursive safety; and (b) its feasible set should be as large as possible to avoid unnecessary performance degradation.
To achieve these properties, we construct the terminal constraint using Hamilton–Jacobi (HJ) reachability analysis. Specifically, we impose the constraint
\begin{equation}\label{eq:safety_value_constraint}
\vfsafe(\state) \geq \epsilon,
\end{equation}
where $\epsilon$ is a small positive constant.
In our theoretical analysis we set $\epsilon = 0$. However, in practice we use a small positive buffer to account for potential numerical inaccuracies in the approximation of $\vfsafe(\state)$, thereby ensuring that the resulting feasible set remains control invariant.
Importantly, the feasible set defined by Eq.~\ref{eq:safety_value_constraint} satisfies the two desired properties described above. We formalize these properties in the following proposition.

\begin{proposition}
    The set $\safeset = \{\state\in\sset|\vfsafe(\state)\geq 0\}$ is the maximal control invariant set under the system dynamics Eq.~\ref{eq:dyn} and state constraint Eq.~\ref{eq:state_constraint}. 
\end{proposition}
\begin{proof}
    This proof follows largely from \cite{borquez_2024_safety_filtering}. Take $\state\in\safeset$. It follows that  $\vfsafe(\state)\geq 0$, and by definition of the safety value function in Eq.~\ref{eq:converged_safety_value_function}, there exists some control signal $\ctrlseq^*$ that keeps the system within $\safeset$ starting from $\state$ indefinitely. Hence, $\safeset$ is control invariant.
    
    Now suppose $\safeset$ is not maximal, then there exists a control invariant set $\safeset'$ such that $\safeset\subset\safeset'$. Take $\state\in\safeset'\backslash\safeset$. By definition of $\safeset$, $\vfsafe(\state)<0$, and it follows immediately from the definition of $\vfsafe(\state)$, $\exists \tdummy\in[\tinit,\infty)$ such that $\tfunc(\traj_{\state,\tinit}^{\ctrlseq}(\tdummy)) < 0$. Hence $\state$ cannot stay within $\safeset'$, and $\safeset'$ is not control invariant. We have reached a contradiction; therefore, $\safeset$ is maximal. 
\end{proof}

Once the constraint in \eqref{eq:safety_value_constraint} is incorporated, we obtain the following MPC formulation, which we refer to as \textit{Safety Value MPC}.
\begin{problem}[Safety Value MPC]\label{prob:sv_mpc}
    \begin{subequations}
    \begin{align}
        &\min_{\ctrlseq} \quad \sum_{\mpcindex=0}^{\ph-1} r(\stateseq_\mpctvar(\mpcindex), \ctrlseq_\mpctvar(\mpcindex)) + \phi(\stateseq_\mpctvar(\ph)) \\ \label{eq:mpc_objective}
        \begin{split}
             & s.t. \quad \stateseq_\mpctvar(\mpcindex+1)  = \ddyn(\stateseq(\mpcindex), \ctrlseq(\mpcindex))  \\
             & \qquad \qquad \quad \forall \mpcindex \in \{0,\ldots, \ph -1\} \\
        \end{split} \\
        & \qquad \ \ctrlconstr(\ctrlseq_\mpctvar(\mpcindex))\geq 0  \ \forall \mpcindex \in \{0,\ldots, \ph -1\} \\
        & \qquad \ \tfunc(\stateseq_\mpctvar(\mpcindex))\geq 0  \ \forall \mpcindex \in \{0, \ldots, \ph -1\}  \\ 
        & \qquad \ \vfsafe(\stateseq_\mpctvar(\ph)) \geq 0 \label{eq:svfunc_constraint} 
        \end{align}
    \end{subequations}
\end{problem}
Solving the Safety Value MPC problem yields a controller that retains the performance benefits of MPC while ensuring persistent safety through the reachability-based terminal constraint. In particular, the safety value constraint guarantees that the terminal state lies within a control-invariant safe set, ensuring that safety can be maintained beyond the planning horizon \textit{regardless} of the MPC horizon length $h$. Consequently, the proposed formulation co-optimizes performance and safety within a single MPC framework.

We now show that the MPC formulation equipped with the safety value function constraint is recursively feasible.

\begin{proposition}[Recursive Feasibility of Safety Value MPC]\label{lemma:recursive_feasibility}
    Suppose the initial state $\state_\tinit \in \safeset$. Then Prob. \ref{prob:sv_mpc} is recursively feasible. 
\end{proposition}
\begin{proof}
   By assumption, $\vfsafe(\state_\tinit) \geq 0$. Then, at initial time $\tinit$, there exists a control signal $\ctrlseq_0^*$ that can keep the system within $\safeset$ indefinitely. Then, the first $\ph$ controls from $\ctrlseq_0^*$ is a feasible solution for Prob.~\ref{prob:sv_mpc}. In other words, Prob.~\ref{prob:sv_mpc} is feasible at time $\tinit$. Let us denote the solution of Prob.~\ref{prob:sv_mpc} at time $\mpctvar=0$ and state $\state_\tinit$ as $\ctrlseq_\tinit$, and we denote the final state that evolves from $\state_0$ under $\ctrlseq_0$ with $\state_{f0}$. 
   
   Apply the first $\ch < \ph$ controls from $\ctrlseq_0$ and we arrive at time $\mpctvar = 1$ and state $\state_1$. Note that by applying controls $\ctrlseq_0(\ch), \ldots \ctrlseq_0(\ph-1)$, the system evolves from $\state_1$ to $\state_{f0}$. Since $\state_{f0} \in\safeset$, there exists a control signal $\ctrlseq_1^*$ that keeps the system within $\safeset$ indefinitely. Then, $\ctrlseq_0(\ch),\cdots, \ctrlseq_0(\ph-1), \ctrlseq_1^*(0), \cdots, \ctrlseq_1^*(\ch-1)$ is a feasible solution for Prob.~\ref{prob:sv_mpc} at time $\mpctvar=1$ and $\state_1$. 

    Repeating the process above, and at time $\mpctvar=k < \tendmpc$ and $\state_k$, Prob.~\ref{prob:sv_mpc} is feasible using the same logic. Hence, Prob.~\ref{prob:sv_mpc} is recursively feasible. 
 \end{proof}  
   
\begin{remark}
    Note that the MPC problem in Prob.~\ref{prob:sv_mpc} is nonconvex in general, and our proposed method does not require convexity. In this work, we employ sequential convex programming \cite{successive_cvx_prog} to solve Prob.~\ref{prob:sv_mpc} efficiently in real time.
\end{remark}

\subsection{Computing Safety Value Function}
To compute the safety value function for high-dimensional systems (e.g., manipulators), we adopt a learning-based method DeepReach \cite{deepreach_mpc}. Unlike \cite{deepreach_mpc}, however, we compute the approximate solution to the HJB-VI (Eq.~\ref{eq:HJBVI}) at the supervision points by solving an \emph{approximate} safety optimal control problem, rather than directly solving the safety optimal control problem in Prob.~\ref{prob:safety_ocp} using gradient-free methods. It is worthwhile to note that DeepReach \cite{deepreach_mpc} learns a \emph{neural approximation} of the safety value function regardless of how the solution to the HJB-VI at the supervision points are computed. In this work, the approximate safety optimal control problem uses the following objective:
\begin{equation}\label{eq:approx_safety_ocp_objective}
J(\state,\tvar,\ctrlseq) = \sum_{\tdummy \in [\tvar,\tend]} e^{-\alpha \tfunc(\traj_{\state,\tvar}^{\ctrlseq}(\tdummy))},
\end{equation}
where $\alpha$ is a large positive constant. This objective heavily penalizes violations of the state constraint along the trajectory, thereby closely approximating the reachability objective in Eq.~\ref{eq: reachability_objective}. Intuitively, trajectories that approach or violate the unsafe region incur exponentially larger penalties, encouraging solutions that remain within the safe set. We solve this approximate safety optimal control problem using a gradient-based solver \cite{jordana_2023_mim_solvers}. In our experiments, we found that solving the approximate safety optimal control problem yields more accurate value function estimates than solving the safety optimal control problem in Prob.~\ref{prob:safety_ocp} with gradient-free methods. We believe that this is due to the summation objective structure being better conditioned than the minimum over time objective structure of the reachability objective in Eq.~\ref{eq: reachability_objective}. 

\section{Simulation Experiment}
\subsection{Experiment Setup}\label{subsec:sim_exp_setup}
In this simulation experiment, we consider a scenario where a manipulator is tasked with moving a heavy object under reduced torque limit, while avoiding the obstacle in the workspace. Such scenarios are of practical importance in manufacturing, as the manipulators are often tasked to carry heavy payloads and actuators inevitably degrade over time leading to reduced torque output. Though motion planning can enable the manipulator to perform agile obstacle avoidance maneuvers, it does not consider the payload nor the torque limit of the manipulator, leading to likely obstacle collision/violation of the state constraint.

In this experiment, we simulate a Flexiv Rizon 10s manipulator with a 14D state space (7 joint position variables and 7 joint velocity variables) and a 7D control space (7 joint torques). The task of the manipulator is to move from a given initial state to the desired end-effector position in the workspace, while avoiding a cylindrical obstacle. Let us denote the joint position and joint velocity by $q$ and $\dot{q}$, and we use the shorthand $FK(\cdot)$ for forward kinematics, mapping a joint position $q$ to the $(x,y,z)$ position of a point on the end-effector. Furthermore, we denote the desired end-effector position by $p_d$. More concretely, the running cost and final cost are given by $r(q, \dot{q}, u) = \phi(q, \dot{q}, u) = ||FK(q) - p_d||_2$, the obstacle constraint is given by $l(FK(q)) \geq 0$.

\noindent \textbf{Baselines and Implementation Details.} We consider two baselines: 1) \textbf{(MPC)} the state-constrained MPC formulation without the additional terminal safety value function constraint, and 2) \textbf{(SB-Filter)} the smooth blending safety filter presented in \cite{ames_2016_cbf_qp, borquez_2024_safety_filtering}. The SB-Filter baseline uses the state-constrained MPC as the nominal controller and performs safety filtering with the safety value function used in our method, and the resulting filtering optimization problem is solved using OSQP \cite{osqp}. Furthermore, we use the SQP solver from \cite{jordana_2023_mim_solvers} designed specifically for MPC and based on the Crocoddyl optimal control library \cite{mastalli_2020_crocoddyl} to implement the MPC controllers used in all 3 methods. We run the experiment from 100 randomly sampled starting states for all methods. All computations are carried out on a desktop with an Intel Core i9 14900KF CPU with 64G of system memory and a NVIDIA GeForce RTX 4090 GPU. The parameters for the experiment are presented in TABLE.~\ref{tab:simulation_experiment_parameters}. 

\begin{table}[h]
    \centering
    \begin{tabular}{lc}
        \hline
        \textbf{Parameter} & \textbf{Value} \\
        \hline
        Task Horizon & 15 s \\
        MPC $dt$ & 0.04 s \\
        MPC Control Horizon & 1 \\
        Solver Maximum Sequential Convexification Iterations & 15 \\
        Safety Value Function Constraint Buffer $\epsilon$ & 0.05 \\
        \hline
    \end{tabular}
    \caption{Parameters for the simulation experiment}
    \label{tab:simulation_experiment_parameters}
\end{table}

\begin{figure}[t]
    \centering
    \includegraphics[width=1\linewidth]{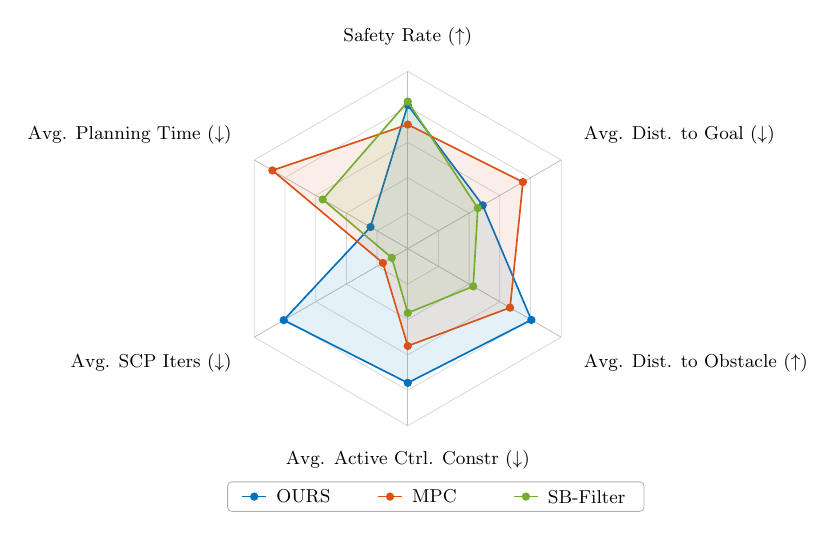}
    \caption{Normalized metrics for the proposed method and the baselines with a planning horizon of 8 in the simulation experiment. Arrows indicate the preferred direction for each raw metric (higher $\uparrow$ or lower $\downarrow$); however, all axes are scaled such that values closer to the outer perimeter represent superior performance on the metric.}
    \label{fig:simulation_experiment_radar_plot}
\end{figure}

\vspace{0.2em}
\noindent \textbf{Evaluation Metrics.}
We consider 6 metrics: 1) \textbf{(Safety Rate)} the percentage of trials that are collision-free; 2) \textbf{(Avg. Dist. to Goal)} the average distance to goal over all time steps and trials, measuring the performance aspect of the methods; 3) \textbf{(Avg. Dist. to Obstacle)} the average distance to the obstacle over all time steps and trials, measuring the safety margin for the methods; 4) \textbf{(Avg. Active Ctrl. Constr.)} the average number of active control constraints over all time steps and all trials, measuring the frequency of actuator saturation; 5) \textbf{(Avg. SCP. Iters.)} the average number of sequential convexification iterations required by the solver to converge over all time steps and trials; and 6) \textbf{(Avg. Planning Time)} the average planning time over all time steps and trials. Note that we compute all the metrics, except the safety rate, using the collision-free trials from each method. 

\subsection{Experiment Results}
We plot the normalized results for the planning horizon $h=8$  in Fig.~\ref{fig:simulation_experiment_radar_plot}. 
In terms of the safety rate, our method outperforms the MPC baseline but performs similarly to the SB-Filter baseline. This is expected as both our method and the SB-Filter baseline, in theory, maintain safety with the safety value function, and the safety value function effectively provides longer safety look-ahead beyond the MPC's planning horizon, leading to better safety rate than the MPC baseline. Beyond baseline safety, our method demonstrates three key advantages that are preferred in practice and for real world deployment: it achieves faster convergence by requiring fewer solver iterations, provides a wider physical safety margin by maintaining a larger average distance to obstacles, and preserves control authority for the low-level tracking controllers by resulting in fewer active control constraints.

On the other hand, our method is outperformed in terms of the average distance to goal and average planning time. It is unsurprising that our method generates trajectories that are on average further away from the goal, as our method maintains larger distance to the obstacle, effectively trading path optimality for safety. For the average planning time, our method is at a disadvantage due to the representation of the safety value function. Since Prob.~\ref{prob:sv_mpc} is solved with a SQP solver, it is necessary to take a forward and a backward pass on the neural network at each convexification step, leading to longer average planning time. Optimizing and compiling our neural network model can help reduce the average planning time. However, the proposed method still remains real-time as we also show in our hardware experiments.  

We visualize the end-effector trajectories over the $x-y$ plane for 2 of the trials in Fig.~\ref{fig:sim_exp_traj_comp_plot}. On the left, we have a trial where our method succeeds but both the baseline fails. The baselines initially choose a more optimal path, in terms of performance (path length), but it eventually leads to an unrecoverable failure. Our method takes a less optimal but ultimately safe path. On the right, we visualize a trial where all 3 methods are safe. Our method stays further away from the obstacle while the baselines almost step on the boundary of the obstacle. 

\begin{figure}[t]
    \centering
    \includegraphics[width=0.9\linewidth]{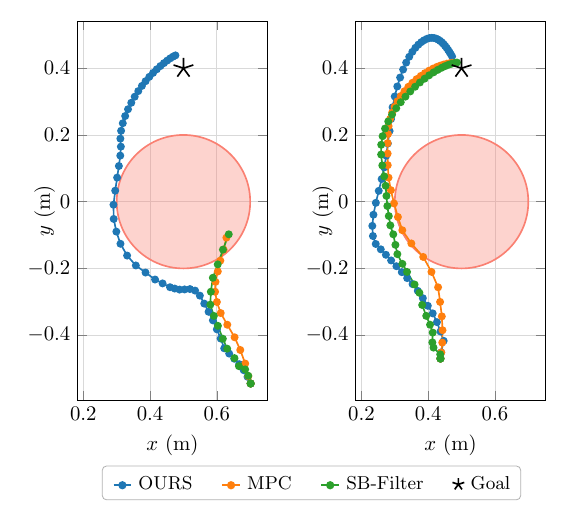}
    \caption{End-effector trajectories visualized over the $x-y$ plane for 2 trials in the simulation experiment.}
    \label{fig:sim_exp_traj_comp_plot}
\end{figure}

We also perform an ablation experiment to investigate the impact of planning horizon on each method in terms of the 6 aforementioned metrics. The results are presented in TABLE.~\ref{tab:ablation_concise}. Here, we focus on the safety aspect of the methods, and the results are visualized in Fig.~\ref{fig:simulation_experiment_ph_ablation}. For each method and planning horizon, we plot the average distance to the obstacle and the safety rate on the $x$ and $y$ axis, respectively. For better visualization, we plot each method with a specific marker shape, and the shade of the marker represents the planning horizon (shade becomes darker as planning horizon increases). We can observe that as the planning horizon increases, the safety rate of the MPC baseline increases. This is expected due to a longer safety look-ahead. On the other hand, the safety rates for our method and the SB-Filter baseline stay roughly unchanged. This observation also aligns with our expectation since both methods in theory keep the system within the safe set $\safeset$ and maintain safety independent of the planning horizon. Another observation is that as the planning horizon increases, all the methods see their average distance to the obstacle decreases, but our method consistently stays further away from the obstacle than the baselines for any planning horizon. 

\begin{figure}[t]
    \centering
    \includegraphics[width=0.9\linewidth]{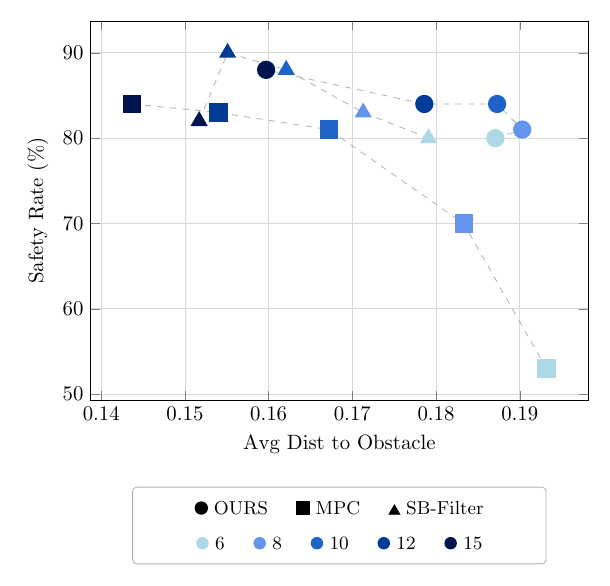}
    \caption{Visualizing the safety rate and average distance to the obstacle for all 3 methods and 5 different planning horizon settings in the simulation experiment.}
    \label{fig:simulation_experiment_ph_ablation}
\end{figure}

\begin{table*}
    \centering
    \scriptsize
    \begin{tabular}{cccccccc}
        \toprule
        Planning Horizon & Method & Safety Rate & Avg. Dist to Goal (m) & Avg. Dist. to Obs. (m) & Avg Active Ctrl & Avg SCP Iters & Avg Planning Time (ms) \\
        \midrule
        \multirow{3}{*}{6} & OURS & \textbf{80\%} & 0.3526$\pm$0.0510 & 0.1871$\pm$0.0248 & 0.750$\pm$0.960 & \textbf{8.689$\pm$6.928} & 33.814$\pm$25.544 \\
         & Vanilla MPC & 53\% & \textbf{0.3498$\pm$0.0447} & \textbf{0.1932$\pm$0.0172} & \textbf{0.675$\pm$0.980} & 13.806$\pm$2.622 & \textbf{10.185$\pm$4.162} \\
         & Smooth Blending Filter & \textbf{80\%} & 0.3804$\pm$0.0551 & 0.1791$\pm$0.0245 & 0.951$\pm$1.208 & 13.875$\pm$2.736 & 17.516$\pm$5.971 \\
        \midrule
        \multirow{3}{*}{8} & OURS & 81\% & 0.4023$\pm$0.0535 & \textbf{0.1903$\pm$0.0269} & \textbf{0.681$\pm$1.033} & \textbf{6.917$\pm$7.214} & 28.926$\pm$27.803 \\
         & Vanilla MPC & 70\% & \textbf{0.3498$\pm$0.0423} & 0.1833$\pm$0.0243 & 0.838$\pm$1.208 & 13.377$\pm$3.152 & \textbf{12.939$\pm$5.884} \\
         & Smooth Blending Filter & \textbf{83\%} & 0.4088$\pm$0.0644 & 0.1713$\pm$0.0257 & 0.977$\pm$1.283 & 13.959$\pm$2.561 & 21.142$\pm$7.071 \\
        \midrule
        \multirow{3}{*}{10} & OURS & 84\% & 0.4298$\pm$0.0540 & \textbf{0.1873$\pm$0.0322} & \textbf{0.728$\pm$1.131} & \textbf{6.637$\pm$7.164} & 29.921$\pm$29.778 \\
         & Vanilla MPC & 81\% & \textbf{0.3791$\pm$0.0441} & 0.1672$\pm$0.0292 & 0.926$\pm$1.376 & 13.364$\pm$3.132 & \textbf{16.020$\pm$7.455} \\
         & Smooth Blending Filter & \textbf{88\%} & 0.4399$\pm$0.0613 & 0.1621$\pm$0.0351 & 1.327$\pm$1.394 & 13.908$\pm$2.654 & 24.805$\pm$8.308 \\
        \midrule
        \multirow{3}{*}{12} & OURS & 84\% & 0.4394$\pm$0.0459 & \textbf{0.1786$\pm$0.0334} & \textbf{1.018$\pm$1.590} & \textbf{6.459$\pm$7.124} & 31.458$\pm$32.103 \\
         & Vanilla MPC & 83\% & \textbf{0.4075$\pm$0.0417} & 0.1540$\pm$0.0298 & 1.186$\pm$1.908 & 13.313$\pm$3.123 & \textbf{18.912$\pm$9.017} \\
         & Smooth Blending Filter & \textbf{90\%} & 0.4455$\pm$0.0607 & 0.1551$\pm$0.0352 & 1.310$\pm$1.636 & 13.731$\pm$2.849 & 28.767$\pm$10.006 \\
        \midrule
        \multirow{3}{*}{15} & OURS & \textbf{88\%} & 0.4655$\pm$0.0512 & \textbf{0.1597$\pm$0.0381} & \textbf{1.178$\pm$2.090} & \textbf{6.580$\pm$7.022} & 34.541$\pm$34.629 \\
         & Vanilla MPC & 84\% & \textbf{0.4405$\pm$0.0472} & 0.1437$\pm$0.0270 & 1.252$\pm$1.976 & 12.855$\pm$3.463 & \textbf{22.910$\pm$11.875} \\
         & Smooth Blending Filter & 82\% & 0.4536$\pm$0.0563 & 0.1517$\pm$0.0376 & 1.213$\pm$1.699 & 13.140$\pm$3.505 & 31.797$\pm$12.810 \\
        \bottomrule
    \end{tabular}
    \caption{Results of all metrics for the simulation ablation experiment}
    \label{tab:ablation_concise}
\end{table*}

\subsection{Discussion}
In this subsection, we analyze the experiment results and discuss several important observations. 

\noindent \textbf{1. Safety guarantees.} The learned safety value function is not perfect, and the feasible set of the safety value function constraint in Eq.~\ref{eq:safety_value_constraint} is not guaranteed to be control invariant. One could seek to provide probabilistic guarantees \cite{lin_2023_verify_neural_tubes, lin_2024_robust_verify_neural_tubes}, but formal guarantees on the learned safety value function are not available. 

\noindent \textbf{2. Safety considerations of our method.} In theory, our method guarantees recursive feasibility and safety. In practice, we cannot achieve the guarantees due to several factors. First, as we have discussed in the previous point, the learned safety value function is not guaranteed to induce a control invariant feasible set. Second, the MPC formulation with the final-time safety value function constraint in Prob.~\ref{prob:sv_mpc} is a non-convex problem. Though we can approximate the solution using sequential convexification techniques, the solution is not guaranteed to satisfy all the constraints, including the safety value function constraint. Third, practically speaking, the safety value function constraint is difficult to optimize. It is well known that the solution, called the viscosity solution, to the HJB-VI in Eq.~\ref{eq:HJBVI} is not differentiable \cite{crandall_viscosity_solution}. Even though we use a smooth neural approximator, namely sinusoidal neural network, to represent the solution to the HJB-VI, the learned safety value function can vary rapidly and pose serious challenges for gradient-based solvers to optimize. In the ablation experiment, 72 out of the total 83 failure trials fail due to the safety value function constraint not being satisfied prior to the state constraint violation, despite the system remaining within the safe set $\safeset$ prior to the violation of the safety value function constraint. This highlights the challenge of optimizing for the safety value function constraint. Lastly, tightening the safety value function constraint by increasing $\epsilon$ in Eq.~\ref{eq:safety_value_constraint} to ensure safety is not always a good idea. As we mentioned in Sec.~\ref{subsec:sv_mpc}, in practice we often try to construct a control invariant set by taking a super-$\epsilon$ level set of the safety value function to mitigate potential inaccuracies of the safety value function. It could be counterproductive with the safety value function constraint, because by increasing $\epsilon$ we are tightening the constraint. Doing so could make the optimization more difficult and lead to solver instability from our experience.

\noindent \textbf{3. Performance considerations of our method.} From the experiment results, our method is generally more conservative than the baselines. This is due to two factors. First, the safety value function constraint introduces additional challenges for optimization as we mentioned previously. Secondly, in this work we opt to use the converged safety value function instead of the time-varying safety value function. The feasible set of the converged safety value function is a subset of that of the time-varying safety value function at any time $\tvar < \infty$, introducing additional conservatism. One could use the time-varying safety value function with $\tvar =$ remaining task horizon at each planning step to reduce conservatism.

\section{Hardware Experiment}

\begin{figure*}[t]
    \centering
    \includegraphics[width=0.9\textwidth]{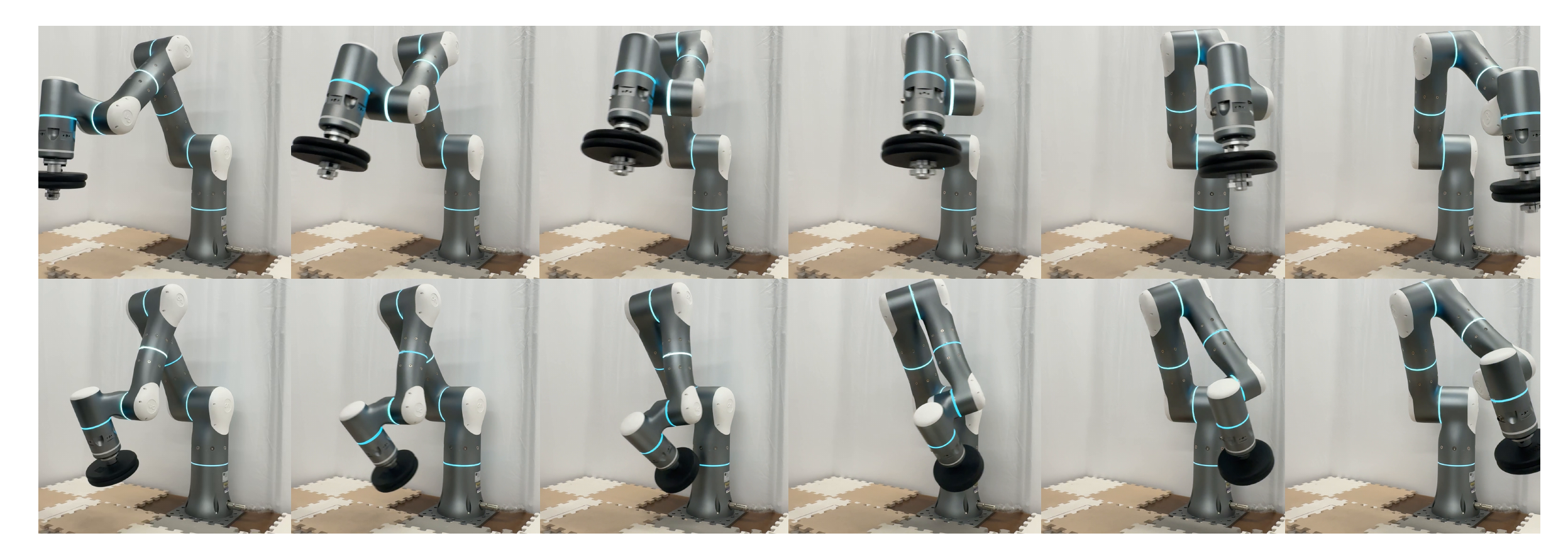}
    \caption{Time-lapse snapshots of 2 trials from the hardware experiments. The obstacle is modeled computationally to ensure collision-free trajectory optimization, but is omitted from the physical workspace to maintain full visibility of the manipulator's motion.}
    \label{fig:hardware_exp_fig}
\end{figure*}

We validate our method on the Flexiv Rizon 10s 7-DOF manipulator hardware platform. We attached two weighted plates with a total mass of 6.8$kg$ onto the robot end-effector. It is worthwhile to mention that the total weight of the payload and the coupling equipment is 7.5$kg$, which is slightly heavier than the 6.8$kg$ payload model we use for training the safety value function. This minor mismatch in the weight of the payload between the model and the hardware did not impact the performance of our method.

The experiment setup is identical to the simulation experiment described in Sec.~\ref{subsec:sim_exp_setup} with minor changes to the experiment parameters, and the parameters for the experiment are provided in TABLE.~\ref{tab:hardware_experiment_parameters}. To facilitate the MPC at 20$Hz$ on hardware and minimize modeling errors, we utilize a low-level PD controller running at 1000$Hz$ to track the desired joint position, joint velocity, and joint torque from the solution of each method. Furthermore, to ensure the safety of personnel and hardware platform, we enforce stricter joint velocity constraints and heavily penalize rapid movements within the MPC formulation, resulting in slower, more controlled hardware operation. All computations for the hardware experiments are carried out on a desktop with an AMD Ryzen 9 8940HX CPU with 64 GB of system memory, and a NVIDIA GeForce RTX 5070 GPU. 

\begin{table}[h]
    \centering
    \begin{tabular}{lc}
        \hline
        \textbf{Parameter} & \textbf{Value} \\
        \hline
        Task Horizon, $T$ & 5.0s \\
        MPC Control Timestep, $dt$ & 0.05s \\
        MPC Planning Horizon, $h$ & 12 (0.6s) \\
        Solver's Maximum Compute Time & 0.04s \\
        Safety Value Function Constraint Buffer $\epsilon$ & 0.05\\
        \hline
    \end{tabular}
    \caption{Parameters for the hardware experiment}
    \label{tab:hardware_experiment_parameters}
\end{table}

We perform 10 trials from different initial states, with snapshots of 2 representative trials — navigating a virtual obstacle — shown in Fig.~\ref{fig:hardware_exp_fig}. Our method achieves a safety rate of $80\%$, whereas the MPC baseline and SB-Filter baseline achieve a safety rate of $30\%$ and $40\%$, respectively. In terms of the average distance to the obstacle and the goal, the distances are fairly close across all 3 methods due to the hardware moving slower for safety reasons. 

\section{Conclusion} 
\label{sec:conclusion}
In this work, we propose a scalable framework that co-optimizes safety and performance for autonomous systems, and we demonstrate our framework on a high-dimensional manipulator platform. Our method incorporates a safety value function as a final-time constraint in the MPC formulation, and it is shown to improve the safety constraint satisfaction of the system in simulation and hardware experiments. However, our method relies on learning-based methods for computing the safety value function for high-dimensional systems, and as a result, our framework currently cannot provide formal safety guarantees. Furthermore, the learned safety value function poses additional challenges, both in terms of computation time and optimization landscape, for our gradient-based solver, potentially impacting the practical performance and safety of the system. We look to address these challenges in future work.  

\bibliographystyle{IEEEtran}
\bibliography{citations} 

\end{document}